\documentclass[twoside]{article}

\usepackage[accepted]{aistats2016}

\usepackage{caption}
\usepackage{subfigure}

\usepackage[utf8]{inputenc}
\usepackage[english]{babel}
\usepackage{float}
\usepackage{blindtext}
\usepackage{amssymb}

\usepackage{times}

\usepackage{amsthm}
\usepackage[authoryear]{natbib}
\usepackage[final]{graphicx}
\usepackage[utf8]{inputenc}
\usepackage{amsmath}
\usepackage{tikz}

\usetikzlibrary{automata,positioning,arrows}
\usepackage[ruled]{algorithm2e}

\def\BState{\State\hskip-\ALG@thistlm}
\newtheorem{theorem}{Theorem}[section]

\begin{document}

\twocolumn[

\aistatstitle{TSEB: More Efficient Thompson Sampling for Policy Learning}

\aistatsauthor{ Prasanna P \And Sarath Chandar \And Balaraman Ravindran }

\aistatsaddress{ IIT Madras\\\texttt{pp1403@gmail.com} \And University of Montreal\\\texttt{apsarathchandar@gmail.com} \And IIT Madras\\\texttt{ravi@cse.iitm.ac.in} } ]

%
%
%
%
%


\begin{abstract}In model-based solution approaches to the problem of learning in an unknown environment, exploring to learn the model parameters takes a toll on the {\it regret}. The optimal performance with respect to {\it regret} or {\it PAC} bounds is achievable, if the algorithm exploits with respect to reward or explores with respect to the model parameters, respectively. In this paper, we propose TSEB, a Thompson Sampling based algorithm with adaptive exploration bonus that aims to solve the problem with tighter PAC guarantees, while being cautious on the regret as well. The proposed approach maintains distributions over the model parameters which are successively refined with more experience. At any given time, the agent solves a model sampled from this distribution, and the sampled reward distribution is skewed by an exploration bonus in order to generate more informative exploration. The policy by solving is then used for generating more experience that helps in updating the posterior over the model parameters. We provide a detailed analysis of the PAC guarantees, and convergence of the proposed approach. We show that our adaptive exploration bonus encourages the additional exploration required for better PAC bounds on the algorithm. We provide empirical analysis on two different simulated domains.%
%
\end{abstract}

\section{INTRODUCTION}

In the standard Reinforcement Learning (RL) framework, the environment with which the agent interacts is modeled as a Markov Decision Process (MDP) and the goal of the agent is to learn a policy such that the cumulative reward it receives is maximized over a finite or an infinite horizon. If the parameters of the MDP are known, then the learning process is straight forward,  and the optimal policy can be learnt by traditional DP-methods \citep{rlbook}. However, in any real life application, the parameters of the MDP are not known {\it a priori}. In such a scenario, the agent can try to directly learn the policy that maximizes the return (model-free learning) or the agent can try to estimate the parameters of the MDP and learn a policy based on the learnt MDP (model-based learning). 

Recently Model-based learning approaches have been receiving increasing attention  \citep{stren,kolter,singhvar,vanroy}. In model-based RL the goal of the agent is two-fold. First, it should estimate the true parameters of the model. Second, it should also behave {\it optimally} during the phase of learning the model parameters. This is yet another instance of exploration-exploitation dilemma in Reinforcement Learning. The agent has to explore to learn the model parameters, but trying to be explorative in improving the belief over the model parameters reduces the performance, i.e., sum of cumulative rewards over a certain number of time-steps. In this approach, the belief over the parameters of the model gets updated as and when the agent receives a sample, $(s_t,a_t,s_{t+1},r_{t+1})$, where $s_t$ is the state the agent is at time $t$, $a_t$ is the action the agent took at time $t$ and $s_{t+1}$ and $r_{t+1}$ are next state and its corresponding reward. As the number of samples increases, the belief converges to the true parameters of the MDP. 

Among model based methods, Bayesian approaches are particularly attractive due to their amenability for theoretical analysis, and the convenient posterior update rule. Much of the recent work has been focused on {\it Thompson sampling} (TS) \citep{ts1933} based approaches both in simpler bandit settings \citep{chapelle,shipra,Shipralinear,Aditya}, as well as the full MDP problem \citep{stren,adiTS}. The Bayesian RL approach proposed in \citep{stren} is an episodic way of incrementally converging to the true parameters of the model. The learning happens in phases, where after each phase, the agent estimates a posterior distribution over the parameters and samples a model for the next episode. The agent solves for an optimal policy with the sampled parameters and uses it to generate trajectories in that episode followed by updating the posterior. This approach of posterior sampling is known as Thompson sampling \citep{ts1933}. The structure of the Bayesian learning as discussed provides a non-zero probability mass over the true model parameters guaranteeing convergence.

Ever since Chapelle and Li \citep{chapelle} discussed the efficacy of TS approaches for reinforcement learning, there have been concerted attempts made to achieve a better theoretical understanding of such approaches. Apart from the results in the bandit setting, Thompson sampling approach for full RL has been shown to work well in practice and has been shown to be regret optimal \citep{adiTS}. However, there are no PAC guarantees in the literature for the Thompson sampling approach. To achieve PAC guarantees we need to encourage {\it more aggressive} exploration than enjoined by the basic TS approach and one way to do that is to use an exploration bonus. For e.g., \citep{kolter} proposed Bayesian Exploration Bonus (BEB) algorithm which added a constant exploration bonus to the problem of solving for an optimal policy in an unknown environment. \citep{kolter} computes a point estimate of the MDP and solves for the optimal policy in every episode. This improves on an another exploration bonus approach, MBIE-EB \citep{mbieb}, in terms of the PAC bounds. Note that in general, adding exploration bonus to the learning agent results in better performance with respect to PAC but may not result in a regret optimal algorithm. 

The primary contribution of our work is TSEB (Thompson Sampling with Exploration Bonus), a Thompson Sampling algorithm that uses an adaptive exploration bonus. As with usual TS approaches, TSEB also maintains a distribution over the parameter space. But the sampling strategy employs an adaptive exploration bonus - when a model is sampled from the distribution in each phase, the rewards of the sampled model are modified. This exploration bonus at that state is related to the current uncertainty in the parameter estimates for the state and this leads to more informative trajectories being generated in each episode. The exploration in most Thompson sampling approaches lead to optimal regret. We show that TSEB encourages additional exploration required for better PAC bounds. To our knowledge this is the first work in the literature to provide PAC guarantees for TS. We empirically show that by appropriately tuning a trade-off parameter we can improve performance with respect to regret as well.

Major contributions of this work are, 
\begin{itemize}
\item Introducing an adaptive, $value$ based, exploration bonus to aid model based learning agent.
\item Providing a tighter PAC guarantee for TS, with exploration bonus. 
\item Theoretically showing the convergence of the algorithm.
\item Empirically showing the inadequacy of TS to be PAC optimal.
\end{itemize}

The rest of the paper is organized as follows. Section 2 describes the preliminaries. We describe the TSEB algorithm in Section 3 followed by theoretical analysis of TSEB in section 4. Section 5 discusses about the regret guarantees of TSEB. In section 6, we experimentally analyse TSEB in 2 domains. Section 7 discusses the related work, and section 8 concludes the paper. 

\section{PRELIMINARIES}

In reinforcement learning, a learning agent interacts with a world modeled as a MDP, $<${\it S,A,R,T,$\gamma$}$>$. The MDP description consists of, {\it S}, the set of all states, {\it A}, the set of all actions, {\it R}, the reward function, {\it R}$:S$x$A\rightarrow\mathbb{R}$, {\it T}, the transition function, $T:S$x$A$x$S\rightarrow[0,1]$ and the discount factor $\gamma \in (0,1)$. The agent has to learn an optimal action mapping, $\pi^*:S\rightarrow A$ that maximizes the cumulative reward over a finite or infinite horizon, $H$. When the model parameters are known, the optimal policy, $\pi^*$, can be obtained by solving the MDP using classical DP-techniques like value-iteration, policy iteration or by optimization methods.

Further in the discussion we will use a metric to bound the distance between the value function of the sampled and the true MDP. The metric is inspired from the homomorphism literature \citep{kbcs}. Consider two different MDPs $M_1$ and $M_2$. Let the $\max$ norm over the difference in their rewards be $K_r$, and in transition be $K_p$, and the {\it range} of the rewards in $M_1$ be $\delta_{r}$. The similarity between the problem of homomorphism and the problem of estimating the closeness of sampled MDP to true MDP is subtle. The structure, state space, $S$ and action space, $A$, remains the same and the $R_2(s)$, reward function in the {\it true} MDP, can be approximated with the samples obtained from the world. Thus, with the given descriptions the difference in the {\it values} of the state in $M_1$ and $M_2$, $v_1$ and $v_2$, which we define as the {\it f-function}, can be bounded by the following expression \citep{kbcs},
\begin{equation}
\label{f-func}
||v_1-v_2||_{\infty}= \frac{2}{1-\gamma}\left[K_r + \frac{\gamma}{1-\gamma}\delta_{r}K_p\right]
\end{equation}

If we have a Dirichlet distribution governing the transition, $K_p$ can be bounded by $\frac{1}{n(s,a)}$ \citep{singhvar}, where $n(s,a)$ is the number of times the state-action pair was observed. We refer Eqn. \ref{f-func} as $f(K_r,\gamma)$. The $\delta_{r}$ is upper bounded by $2$ in a normalized bounded rewards setting ($r_t\ \in$ [-1,1]), 
\begin{equation}
\label{HMbound}
f(K_r,\gamma)=\frac{2}{1-\gamma}\left[K_r + \frac{2\gamma}{(1-\gamma)\underset{s\in\mathcal{S}, a\in\mathcal{A}}\min n(s,a)}\right]
\end{equation}

$K_r$ is estimated from the difference between the expected reward sampled in an episode $e$ for an arbitrary state $s$ to the empirical mean of the rewards sampled with samples obtained till episode $e-1$ and time-step $t-1$ in that episode.  This measure, {\it f-function}, provides a measure of variance of the sampled MDP. This is an {\it unbiased} estimate of the distance from the true reward.

Let $\hat{R}_n$ be the sample average constructed with {\it n} samples, across the episodes. As $n\rightarrow\infty$,
\begin{equation}
\frac{1}{n}\sum_n^{\infty}\hat{R}_n(s,a)-\mathbb E[R(s,a)]\rightarrow 0
\end{equation}

The above expression shows that the computed sample average is an unbiased estimate of the expectation of the random variable $R(s,a)$, reward of state s.

\section{TSEB ALGORITHM}

TSEB is an episodic approach, where the incrementally sampled model grows closer to the true model. Solving the converged model gives us a near optimal policy. From the problem as posed, it is intuitive to understand that the agent has to learn the true model to converge to an optimal policy, $\pi^*$. TSEB has a modified Bellman update that considers the exploration bonus. The reward, thus, is a convex combination of the reward obtained from the sampled world, and the exploration bonus computed for that state $\rho(s)$ (or  $\rho(s,a)$) in that episode. The Bellman update will be,
\begin{equation}
V(s)=\lambda R(s,a)+(1-\lambda)\rho(s,a)+\gamma\sum_{s'} P(s'|s,a)V(s')
\end{equation}
\begin{equation}
\rho(s,a)=\frac{f_s(K_r,\gamma)}{n(s,a)}
\end{equation}
where $f_s(k_r,\gamma)$ provides an upper bound on the difference in the value of the state $s$ between the true and sampled MDP. {\it n(s,a)} is the number of times {\it s} was visited.

The $f_s$ term is similar to the $f(K_r,\gamma)$ defined in the preliminary section except that instead of $||.||_{\infty}$ it will be computed for that particular state, $s$ ({ \it $\underset{s}\min\ n(s,a)$ gets replaced by the {\it n(s,a)}}). 

\begin{equation}
f_s(K_r,\gamma)=\frac{2}{1-\gamma}\left[K_r + \frac{2\gamma}{(1-\gamma)n(s,a)}\right]
\label{ffun6}
\end{equation}

The algorithm follows a greedy policy and takes the max action with respect to the modified Bellman update. As $\rho(s,a)$ decays with every visit, the agent explores the state space adaptively. The updates from the sampled trajectories help the distribution to narrow its belief- reducing the variance over the distribution. As the agent samples from those states that are useful, by following a greedy policy, more often after a few episodes, the sampled MDP might not be $\epsilon$ close to the true MDP, or cannot be guaranteed. But, the parameters of the better rewarding states will be close to the optimal, thus providing an $\epsilon$-optimal policy. Note that TSEB learns optimal policy for states which are $useful$ and the notion of $useful$ states evolve over the episodes. Thus, even though TSEB does not learn optimal policy for the true MDP, it learns a near-optimal policy which is more close to the optimal policy in states that will be often visited by the agent.

The linear decay of the exploration bonus makes the exploration bonus to become insignificant either when the parameters are closer to the true MDP or if the number of visits is large.

    \renewcommand\baselinestretch{1}
    \begin{algorithm}[t]
    
    \caption{Thompson Sampling with Exploration Bonus (TSEB)}\label{tseb}
    
    \SetKw{return}{return}
    \SetKw{from}{from}
    \SetKw{downto}{downto}
    \SetKw{step}{step}
    \KwIn{Parameter Space $\Theta$, prior over $\Theta$, action space $\mathcal{A}$, state space $\mathcal{S}$, and $\gamma$.\newline \hspace*{-1.1cm}\textbf{Define: }E ~~$\triangleright$ Number of episodes.\newline $T_e$ ~~$\triangleright$ Number of time-steps in episode $e$.\newline $\rho$ ~~$\triangleright$ Exploration Bonus.\newline $r_t^e$~~$\triangleright$ reward at time step $t$ in episode $e$. \newline $r$~~$\triangleright$ reward obtained by taking an action.\newline $V_e$~~$\triangleright$ Value function in episode $e$. \newline $R_e$~~$\triangleright$ Reward function sampled for episode $e$. \newline $P_e$~~$\triangleright$ Transistion function sampled for episode $e$. \\ \\} 
    
    \KwOut{policy $\pi$}
  
   \For{e in range (E)}{
     $M_{\theta}\gets$Sample $R_e\ and\ P_e$ from posterior\\
     $V_e\gets$ Solve\_for\_$\pi^*(M_{\theta}$)\\
    \For{t in range ($T_e$)}{
    $\pi(s_t^e)\gets\max_{a\in A} ((\lambda R(s_t^e,a) +(1-\lambda)\rho(s')$\\ $~~~~~~~~~~~~~~~~~~~~~~~~~~~~~+\gamma\sum_{s'}P_e(s'|s,a)V_e(s'|s,a)
    )$\\
    r$\gets$take\_action($\pi(s_t^e)$)\\
    n($s_t^e$)$\gets$ n($s_t^e$)+1\\
    n($s_t^e,s_{t-1}^e$)$\gets$ n($s_t^e,s_{t-1}^e$)+1\\
    r($s_t^e$)$\gets$r($s_t^e$)+$\frac{1}{n(s_t^e)}\left[r-r(s_t^e)\right]$\\
    $\rho(s_t^e)\gets  \frac{\rho(s_t^e) + f(s_t^e)}{n(s_t^e)}$\\
   }  
   Update the posterior: $\pi_{t+1}(d\theta)\propto p(S_t,A_t,R_t,S_t')\pi(d\theta)$
}
        \end{algorithm}

TSEB, unlike the most other previous algorithms (Except \citep{singhvar}) uses the uncertainty in the estimates to structure the exploration bonus. This helps us model real world systems, which have inherent uncertainty. Exploiting the inherent uncertainty in the events to decide on exploration is the key feature of our work. The exploration bonus entails the PAC guarantees of the algorithm. Further theoretical analysis shows us that the bound is indeed tighter than in \citep{singhvar}. The exploration bonus is computed here even more cleverly, thus avoiding integrating over the parameters, compared to \citep{singhvar}. Also there is a principal difference with \citep{kolter}, wherein the exploration of the agent is concentrated around the uncertain region and the uncertainty is not assumed to be uniform over the world. This formulation can also be extended to provide exploration bonus {\it a priori}. Apart from the prior over the model parameters, prior over the exploration bonus helps in learning the model faster and better. We don't analyze {\it a priori} exploration bonus in this work.

\section{THEORETICAL ANALYSIS OF TSEB}

PAC analysis provides an upper bound on the number of sub-optimal steps of an asymptotic agent that is required for the algorithm to converge to an $\epsilon$-optimal solution with probability $1-\delta$. Let us assume the algorithm requires a set of samples, {\it M}. Let each sample be $(s_t,a_t,s_{t+1},r_{t+1})$. Though, TS theory gives us regret guarantees at $\mathcal{O}$($\log$T) \citep{adiTS}, where $T$ is the number of time-steps, we don't have a notion of PAC-bound for TS. This is primarily because the algorithm has to be explorative to learn the model parameters for it to be PAC-optimal. The conundrum here is, if the algorithm is explorative its regret will be worse. Hence, the greedy action selection doesn't let the agent to be explorative. In TSEB with the addition of exploration bonus and thereby skewing the sampled MDP in a way that would make the exploration as part of rewards, the agent can still be greedy with the action selection and provide a sample guarantee. The linear decay of exploration bonus helps the agent converge to the optimal policy like in the TS setting.

\begin{theorem} Exploration using the defined exploration bonus , $\rho$, leads to a monotonic convergence of sampled parameters to the optimal parameters. 
\end{theorem}

\begin{proof} This theorem essentially says that adding exploration does not affect the monotonic convergence of Thompson Sampling. To prove this theorem, it is sufficient to show that the exploration bonus leads to a monotonic convergence of the f-function. The $f$-function defined in Eqn. \ref{ffun6} for every episode, by way of doing posterior sampling, decreases with samples. As the $f$-function is an unbiased estimate of the difference in the value of the true and the sampled MDP (Eqn. \ref{f-func}), the decrease in the {\it f-function} indicates that the sampled model parameters are closer to the true MDP. The {\it f-function} will converge to an $\epsilon'$ such that $\epsilon'\geq0$. For any number of samples further, the sampled MDP lies within the $\epsilon$-ball of the true MDP. To bound the saturation point, let us estimate the rate of change of $f$ with respect to timestep, $\Delta f$,
\begin{equation}
\Delta f\leq\frac{2}{1-\gamma}\left[\Delta K_r + \frac{\gamma}{1-\gamma}c \frac{1}{n(s,a)}\right]
\end{equation}

Since the rewards are bounded,  rate of change of {\it range} can be bounded by a constant {\it c}, $0<c\leq$ 2. At optimum, the first order derivative vanishes. Hence,
\begin{equation}
\Delta K_r\leq\frac{c\gamma}{(1-\gamma)n(s,a)}
\end{equation}

The sum over differences across all the states gives an upper bound on the {\it true} distance between the sampled and {\it} actual MDP. Hence, in an episode this is,
\begin{equation}
\sum_{s,a\in\mathcal{S}\times\mathcal{A}}\Delta K_r=\sum_{s,a\in\mathcal{S}\times\mathcal{A}}\frac{c\gamma}{(1-\gamma)n(s,a)}
\end{equation}

Now, with $\mathcal{S}$, the cardinality of the set of states, and $\mathcal{A}$, the cardinality of set of actions, this can be upper bounded by,
\begin{equation}
\sum\Delta K_r\leq\frac{\mathcal{SA}c\gamma}{(1-\gamma)n_{min}(s,a)}
\end{equation}

Let the sum over differences be denoted by $\tau$, then,
\begin{equation}
\tau \leq \frac{\mathcal{SA}c\gamma}{(1-\gamma)n_{min}(s,a)}
\label{conv}
\end{equation}

Eqn. \ref{conv} states that the $\min_{s\in\mathcal{S}}$n(s,a) is inversely related to $\tau$. And, $\tau$ is directly proportional to $f$. As we don't discard the samples, the $n_{min}(s,a)$ increases monotonically thus letting the $\tau$ to decrease monotonically. The saturation of the upper-bound on the distance, {\it f-function}, provides a formal guarantee of the convergence of TSEB.
\end{proof}

\textbf{PAC-MDP:} An RL algorithm is said to be PAC-MDP, if for any MDP, M, $\epsilon > 0$, $0< \delta <1$, the sample complexity of the algorithm is bounded by some function f that is polynomial in {\it S, A, $1/\epsilon$, $1/\delta$, and $\frac{1}{1-\gamma}$}, with probability at least 1-$\delta$.

\begin{theorem} After M = $\mathcal{O}\left(\frac{SAf_0(K_r,\gamma)}{\epsilon^2} \right)$ steps, TSEB converges to an $\epsilon$-optimal value function with probability 1-$\delta$.
\label{pac}
\end{theorem}

{\it Outline of the Proof.} Let $p(u)$ be the expected probability of selecting an action that will lead the agent to an unexplored state. Let us define a positive non-zero number, $T$, which is the number of time steps in each episodes such that the expected number of visits to unexplored states is {\it at least} 1. Let there be a finite positive integer $k$, the number of times a state has to be visited for its exploration bonus to become insignificant implying that the state has been explored.
 
With the above notations, the number of episodes required to converge to the true MDP parameters will be {\it kSA}, where $S$ and $A$ are the cardinality of state and action sets. We define the samples required for the algorithm to be optimal as {\it kSAT}. The expression has {\it k}, and {\it T} which are not known. The expressions obtained in this section, for the sample complexity, maps to {\it k} and {\it T} indirectly.
 
By showing the $||.||_{\infty}$ of difference between the true MDP and the sampled MDP monotonically decreases with every sample and f-function provides a finite length converging sequence, we can compute the total sample complexity, $M$. We consider variance-based concentration measure, as it is more applicable for deriving the bounds for TSEB, and provides a sharper concentration measure than the Chernoff-Bound used in the analysis of UCB like algorithms. Let $X_1,X_2,...,X_n$ be independent random variables with $\mathbb{E}[X_i] = \mu$ and $Var[X_i] = \sigma^2$. Then for $\epsilon >$0,

\begin{equation}
P\left(\frac{1}{n}\sum_{i=1}^n X_i \geq \mu + \epsilon\right) \leq e^{\frac{-\epsilon^2}{4\sigma^2}}
\end{equation}

This is an extension of the Chernoff bounds \citep{chernoff1952} in a known variance setting.

\begin{proof} 
From {\it Variance bounds} definition, for a sample of reward sequence from a single state, $(R)_i$. Let $\mathbb{E}[R_i] = R^*$, and $Var[R_i] = \sigma^2$
\begin{equation}
P\left(\frac{1}{n}\sum_{i=1}^n R_i \geq R^* + \epsilon\right)  =  e^{\frac{-\epsilon^2}{4\sigma^2}} = \delta \end{equation}

\begin{equation}
\sigma^2= \frac{\epsilon^2}{4\log\frac{1}{\delta}}
\end{equation}

The above equation expresses the relation between $(\epsilon,\delta)$ and ($\sigma^2$). $\sigma^2$ in the above equation is the summation of differences in {\it value} of states between the true and sampled MDP. This can be upper bounded by $S*A*f(K_r,\gamma)$.

Further we need to establish that the exploration bonus decays with the variance of the model parameters. By definition, the exploration bonus $\rho(s,a)$ is a cumulative sum of differences between the sampled state parameters and an unbiased estimate of the  parameters. This,
\begin{equation}
\rho(s,a) = \frac{1}{n}\sum_n|\hat{\mathbb{E}}[\theta_{s,a}] - \theta_{s,a}'|
\label{1norm}
\end{equation}

where, 
$\theta_{s,a}$ is the augmented notation for the state's parameter, $\hat{\mathbb{E}}[\theta_{s,a}]$ is an unbiased estimate of the mean, and $\theta_{s,a}'$ is the sampled parameter.

The variance of the estimate will be {\it 2-norm} of the estimate above, Eqn. \ref{1norm}, and we look at cumulative {1-norm}. The difference occurs in the magnitude of the convergence rate, but the point of convergence remains the same. Hence, we use a variance based complexity bound to upper bound the number of samples.

For a tighter $\epsilon$, variance has to be smaller. This implies that $f(K_r,\gamma)$ has to be smaller. As repeated sampling of trajectory decreases the variance, this is set as an adaptive exploration bonus to the agent. Now, with $n$ being the number of visits to a state $s$, and $f_o(K_r,\gamma)$ being the expected initial distance of the sampled MDP from the true MDP with respect to the prior,
\begin{equation}
\frac{f_0(K_r,\gamma)}{n}\geq \frac{\epsilon^2}{4\log\frac{1}{\delta}}
\end{equation}

The upper bound on the number of visits to an individual state is given by,

\begin{equation}
n \leq \frac{4f_0(K_r,\gamma)}{\epsilon^2}\log\frac{1}{\delta}
\end{equation}

The total sample complexity, M, for ($\epsilon,\delta$) guarantee on the converged MDP, with {\it S} and {\it A} being the cardinalities of set of states and actions, is bounded by,
\begin{equation}
M \leq \frac{4SAf_0(K_r,\gamma)}{\epsilon^2}\log\frac{1}{\delta} 
\end{equation}

\begin{equation}
 M  = \mathcal{O}\left(\frac{SAf_0(K_r,\gamma)}{\epsilon^2} \right)
 \label{totsam}
\end{equation}

Eqn. \ref{totsam} shows that the upper bound on the sample complexity is dependent on the initial estimates of the model. $M$, the total sample complexity is adaptive, as it is a function of the distance between the sampled MDP and the true MDP. The bound, hence, is adaptive and theoretically better than the earlier bounds on sample complexity for a PAC-MDP.
\end{proof}

\begin{table}[ht]
\centering
\caption{The table shows the existing PAC bounds for a model based learning setting.}
\begin{tabular}{|c|c|}
 \hline
 \textbf{Algorithm}& \textbf{PAC- Bounds}  \\
 \hline
  MBIE \citep{mbieb}& $O\left( \frac{S^2AR_{\max}^5ln^3\frac{SAR_{\max}}{(1-\gamma)\epsilon\delta}}{(1-\gamma)^6\epsilon^3} \right)$\\
 \hline
 BEB \citep{kolter}& $O\left( \frac{SAH^6}{\epsilon^2}\log \frac{SA}{\delta}\right)$ \\
 \hline

 Variance Based \citep{singhvar}& $O\left(\frac{\gamma^2S^4A^2}{\delta\epsilon^2(1-\gamma)^2}\right)$\\
 \hline
 TSEB & $O\left(\frac{SAf_0(K_r,\gamma)}{\epsilon^2}\log \frac{1}{\delta}\right)$\\
 \hline
\end{tabular}

\label{existingbounds}
\end{table}

Table \ref{existingbounds} lists all the existing PAC bounds for model based learning setting. Note that TSEB is the only Thompson Sampling algorithm in this list. PAC bound for TSEB is better than MBIE and variance based method while $f_0$ term has $SA$ term which makes this bound higher when compared to BEB. However, TSEB also performs better in regret which is not guaranteed with BEB.

\section{ARGUMENTS ON REGRET}
The discussion so far elucidates the PAC guarantees offered by the TSEB algorithm. The claim of the algorithm being not going worse in {\it regret} has not been addressed so far. Following a greedy policy from the sampled MDP is not very different from the TS approach. The parameters sampled in every episode grow closer to the true model as discussed empirically and theoretically in earlier corresponding sections. As the TSEB agent acts greedily with the sampled model parameters and the model parameters converge, the agent after a certain number of episodes will be acting optimally with the true parameters, $\theta^*$. Because, the greedy policy in the $M_{\theta^*}$ will be an optimal policy $\pi^*$.

The exploration bonus, a linearly decaying component in the modified Bellman update, will become insignificant even if it doesn't become zero. This ensures that TSEB behaves like pure Thompson Sampling after sufficient exploration. Let us define regret at any arbitrary step $i$, $\Delta_i$, as
\begin{equation}
\Delta_i = \mu^* -\mu_i
\end{equation}

Where, $\mu^*$ is the expected reward by taking the optimal action and $\mu_i$ is the average reward obtained at a step $i$.
The expected regret $E(R)$,
\begin{equation}
E(R) = \sum_i\Delta_iE(T^i)
\end{equation}

where, $T^i$ is the number of sub-optimal steps in an episode and $\Delta_i$, the expected regret in episode $i$, different from the previous definition. From previous sections, we can observe that the $T_i^n$ is a converging sequence and so is $\Delta_i$, because of the greedy policy that is mandated in the TS algorithm.

From the algorithm, it is clear that it behaves like the {\it true} Thompson Sampling algorithm after a point when the exploration bonuses becomes numerically insignificant. The sub-optimal steps taken by the agent falls into the two cases,

\begin{itemize}
\item When the sampled parameters are off from the {\it true} model parameters and the agent takes a greedy action.
\item Taking an action that is not the optimal action with respect to the sampled MDP.(May be due to the uncertainty in the action selection.)
\end{itemize}

The recent work on regret in parameterized MDP \citep{adiTS} is a major contribution to the regret analysis of the full RL Thompson sampling approach. The arguments for the regret analysis of TSEB can be done similar to the TSMDP, but varies in the additive constant term. The big-oh notation of the regret makes it insignificant and hence the same analysis holds.

\section{EMPIRICAL ANALYSIS}

In this section, we experimentally analyze the performance of TSEB. We run experiments in two simulated domains, {\it Chain world} \citep{kolter} and {\it Queuing Domain} \citep{adiTS}. The aim of the experiments is to experimentally validate the claim of convergence of the belief and analyze the algorithm under different values of the trade-off parameter, $\lambda \in $[0,1].

\subsection{CHAIN WORLD} 
The chain domain has 5 states and 2 actions $a$ and $b$. The agent can take both the actions from any state. With probability 0.2 the agent takes the opposite action than the one selected. The transitions and reward are shown in the figure. The first state has a stochastic reward, from a Gaussian $\mathcal{N}\left(0.2,0.5\right)$. The optimal policy with $\gamma=0.8$ is to take action $a$ in all the states. The algorithm is experimented on different values of the trade-off parameter (Table \ref{chaintab}). The analysis shows better performance (cumulative sum of rewards) on every non-zero value. This is intuitive, because when $\lambda$ is 0 the algorithm behaves only to reduce the variance and ignores the rewards obtained in the world. This behavior is expected. But, the performance increases with increase in $\lambda$ and decreases after 0.5. The performance has high variance and is inconsistent when $\lambda=1$; this is the TS case. $\lambda$ = 0.1 has the maximum cumulative reward in this case.
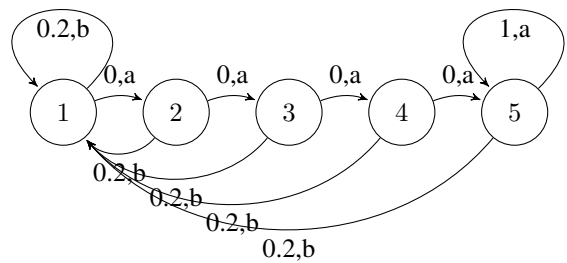
\begin{figure}[ht]
\centering
\begin{tikzpicture}[>=stealth',shorten >=1pt,auto,node distance=1.5cm]

\centering
  \node[state]         (q1)  {$1$};
  \node[state]         (q2) [right of=q1] {$2$};
  \node[state]         (q3) [right of=q2] {$3$};
  \node[state]         (q4) [right of=q3] {$4$};
  \node[state]         (q5) [right of=q4] {$5$};

  \path[->]          (q1)  edge   [bend left=20]   node {0,a} (q2);
  \path[->]          (q1)  edge   [loop]           node {0.2,b} (q1);
  \path[->]          (q2)  edge   [bend left=20]   node {0,a} (q3);    
  \path[->]          (q3)  edge   [bend left=20]   node {0,a} (q4);
  \path[->]          (q4)  edge   [bend left=20]   node {0,a} (q5);
  \path[->]          (q5)  edge   [loop]           node {1,a} (q5);
  \path[->]          (q5)  edge   [bend left=50]   node {0.2,b} (q1);
  \path[->]          (q4)  edge   [bend left=50]   node {0.2,b} (q1);
  \path[->]          (q3)  edge   [bend left=50]   node {0.2,b} (q1);
  \path[->]          (q2)  edge   [bend left=50]   node {0.2,b} (q1);

\end{tikzpicture}
\caption{Chain world domain}

\end{figure}

\begin{table}[t]
\centering
\caption{Average cumulative reward under different $\lambda$ in Chain World}
\vspace{0.2in}
\begin{tabular}{|c|c|}
\hline
\hline
 $\lambda$& \textbf{Average cumulative reward}  \\
 \hline
 0.0 &1382.80\\ 
\textbf{0.1} &\textbf{1963.74}\\
0.2 &1951.72\\
0.3 &1944.65\\
0.4 &1954.06\\
0.5 &1956.22\\
0.6 &1955.14\\
0.7 &1953.99\\
0.8 &1940.77\\
0.9 &1934.99\\
1.0 &1942.63\\
\hline
\end{tabular}

\label{chaintab}
\end{table}

Further we analyze the convergence of the model parameters in Fig.\ref{tf-plot} and Fig.\ref{f-plot}. We plot the {\it f-function} against the {\it Episodes}. The graphs explain that the posterior sampling with exploration bonus converges faster. When $\lambda=0$, the plot shows that the algorithm converges to an inferior model. The inferiority in the model corresponds to the higher f-value. The f-value for $\lambda$ = 0.5 converges to a much better model. This can be argued because the agent considers both the variance in the model parameters as well as the reward obtained in the true world to be maximized, thus converging to a better model.

\begin{figure}[H]
\centering
\subfigure[Convergence of f-function, for different $\lambda$ values in Chain world.]{
\includegraphics[scale=0.21]{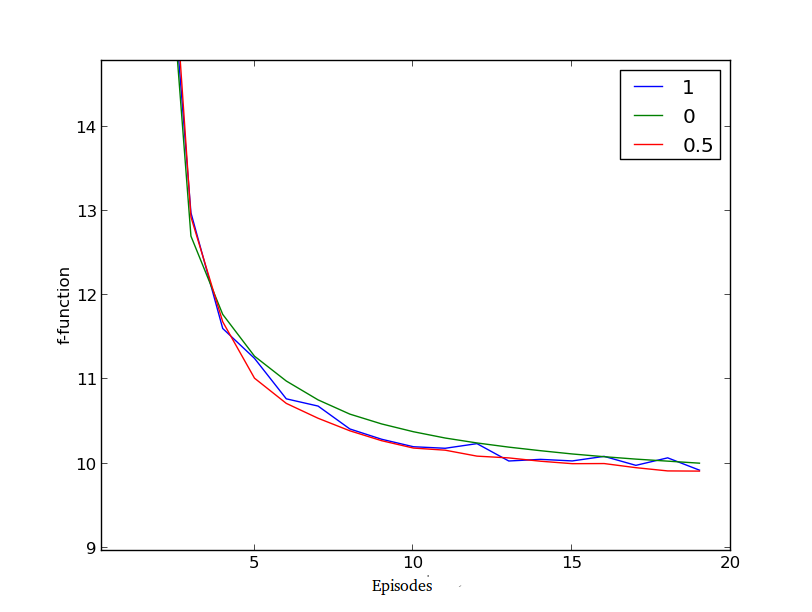}
\label{tf-plot}
}
\subfigure[Convergence of upper bound on f-function for different $\lambda$ values in Chain world.]{%
\includegraphics[scale=0.3]{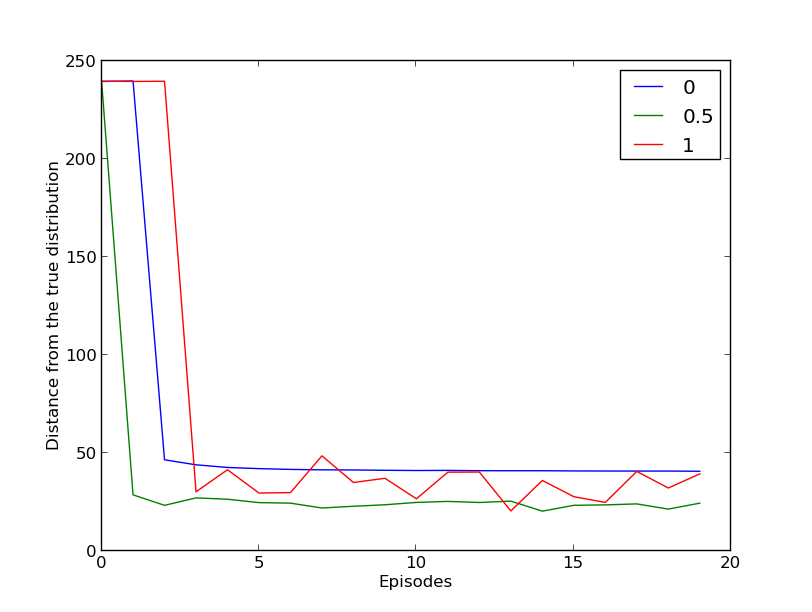}
\label{f-plot}
}
\\
\subfigure[Average regret for Chain World domain.]{%
\includegraphics[scale=0.3]{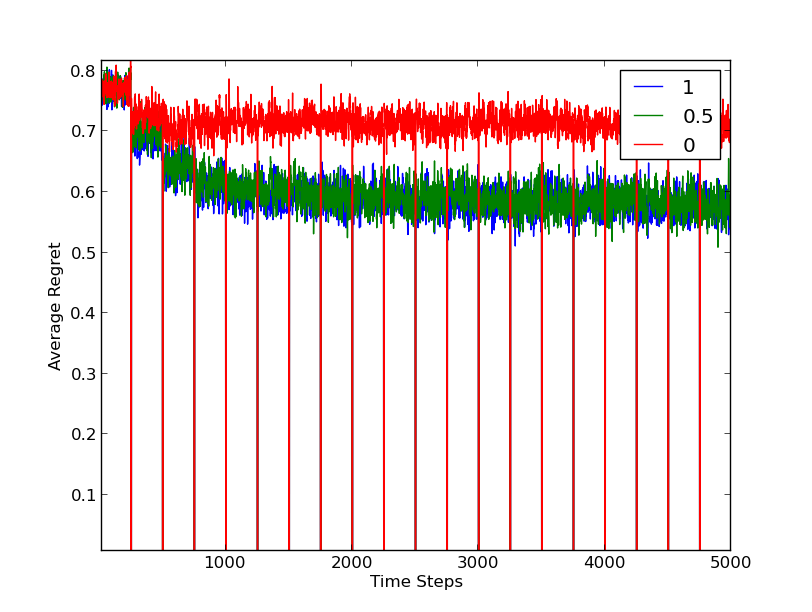}
\label{avreg}
    }

\caption{Chain world Domain}

\end{figure}

Thompson sampling, which is a special case when $\lambda$=1, keeps oscillating and doesn't converge. This is because of the lack of exploration. The graph relates the distance between the sampled MDP and true MDP to the number of samples. As the agent in TS set up acts greedily, the exploration of the agent is poor. The agent has to explore to converge to the true model parameters. TS, being regret optimal always chooses the greedy action and doesn't explore the state-space well.  Hence, the poor PAC-guarantees of TS is experimentally validated. Similarly, the better PAC guarantees that can be obtained by inducing exploration bonus is validated as well.

Figure \ref{avreg} shows the average regret for all three cases. We see that $\lambda = 0.5$ which results in better average reward (from Table \ref{chaintab}) is not doing worse in regret when compared to pure TS setting ($\lambda=1$). 

\subsection{QUEUING WORLD}
We analyse the TSEB algorithm with different $\lambda$ values (Table \ref{qu}) in the Queuing world defined in \citep{adiTS}. The states of the MDP is simply the number of packets in the queue at any given time, i.e.,S=$\{$0,1,2,...,50$\}$. At any given time, one of 2 actions: Action 1 (SLOW service) and Action 2 (FAST service) may be chosen, i.e., A=\{1,2\}. Applying SLOW (resp. FAST) service results in serving one packet from the queue with probability 0. 3 (resp. 0.8) if it is not empty, i.e., the service model is Bernoulli($\mu_i$) where $\mu_i$ is the packet processing probability under service type i= 1,2. Actions 1 and 2 incur a per-instant cost of 0 and -0.25 units respectively. In addition to this cost, there is a holding cost of -0.1 per packet in the queue at all times. The system gains a reward of +1 units whenever a packet is served from the queue.

\begin{table}[ht]
\centering
\caption{Average cumulative reward under different $\lambda$ in Queuing World}
\begin{tabular}{|c|c|}
\hline
\hline
 $\lambda$& \textbf{Average cumulative reward}  \\
 \hline
 0.0 &-5050\\ 
0.1 &-5061\\
0.2 &-5038\\
0.3 &-5048\\
0.4 &-5051\\
0.5 &-5042\\
0.6 &-5040\\
0.7 &-5062\\
0.8 &-5038\\
0.9 &-5026\\
\textbf{1.0} &\textbf{-5023}\\
\hline
\end{tabular}

\label{qu}
\end{table}


The comparison in Table. \ref{qu} shows the cumulative reward for different settings. $\lambda=1.0$, the regret optimal case, outperformed the others. This is because the model didn't have much variance in the parameter, so the learning was faster. Hence, the regret optimal way was better than the rest. 

The two worlds provide two different scenarios: one in which the difference between the performance with different $\lambda$ values is large (Chain Domain), two in which the difference is less (Queuing Domain). Both the experiments suggest that the combination of exploration bonus and the true rewards in the MDP provides a better performance. As the Chain domain doesn't offer negative rewards to the agent, exploration as well does pay off well for the agent. But relying only on exploration bonus, $\lambda$=0, doesn't let it converge to the optimal policy. Hence, the reason the agent accumulates better reward when it is not being exploration centric. On the other hand, in the Queuing domain, the agent receives negative rewards as well, this doesn't aid the agent being over explorative, and the variance in the model parameters are less as well. Hence it accumulates better cumulative reward when it is {\it regret} optimal. These two experiments suggest a heuristic to tune $\lambda$. $\lambda$ can be dynamically adapted with respect to unbiased variance in the reward parameter estimate. We leave this as a future work. 


\section{RELATED WORK}

{\it Optimism in the face of uncertainty}, is an appreciated approach and reasonably widely applied in practice. The approach over-estimates the {\it state-value} or {\it action-value} estimates with some heuristic to aid in exploration of the agent. In \citep{kael}, an algorithm proposed as {\it Interval estimation Q-learning} (IEQ), the action with the highest upper bound on the underlying Q-value gets chosen. This work also asserts that the gradual decay of the over-estimation lets the agent converge to the optimal policy. This has been followed in approaches as early as UCB\citep{auer}, where the empirical mean, $\hat{\mu}_i$, of an arm {\it i} is over-estimated by the confidence interval of the estimated mean. And, for solving an MDP, the UCRL \citep{ucrl} takes an approach inspired by the UCB technique for over estimation to aid exploration. This provides a logarithmic regret bounds in an MDP setting. In an unknown environment setting, the variance based approach to over estimate the {\it value} of a state to aid in exploration was proposed in \citep{singhvar}, but it is not a TS approach. 

Quite a few approaches have addressed the sample complexity issue in RL. But, while being sample efficient the regret gets worse. And, hence PSRL is better when regret optimal learning is needed. Also, the theoretical guarantees of TS have not been analyzed until recently \citep{shipra}. Similar guarantees, though, were not extended to the PAC setting. Recently, \citep{adiTS} gave a regret analysis of TS in full MDP setting that is logarithmic in $T$, the time-steps. \citep{vanroy} highlighted an information-theoretic analysis of TS, giving a better regret bound, considering the entropy of the action-distribution. In the last decade, parameter estimation was extended for the MDP setting; an episodic way of solving for the model estimation in unknown environment \citep{stren}.  

More recently, \citep{kolter} proposed Bayesian Exploration Policy (BEB) algorithm which added a constant exploration bonus to the standard (non-Thompson sampling based) Bayesian RL. This is improved upon the MBIE-EB \citep{mbie}, an interval based exploration bonus algorithm, by increasing the decay rate. \citep{kolter} states that the Bayesian approach cannot have a PAC solution if it doesn't encode an exploration bonus. So, BEB first proposed a bound on the samples, which is the first PAC-analysis of the Bayesian RL. In line of \citep{stren} BOSS, {\it Best of Sampled Sets} \citep{boss} that samples multiple models and merges them. The framework then runs trajectories on the derived MDP. It has a constant B, the number of visits for the agent to know a state's parameters. Note that TSEB can be extended to BOSS setting by sampling multiple MDPs, and following TSEB exploration bonus.

From the literature, it is evident that quite a few approaches were looked at in solving for an optimal policy. The most recent of them include computing the mean MDP \citep{kolter} and ML MDP \citep{singhvar}. As the two approaches compute a point estimate, it is theoretically very likely that the probability mass over the true model parameters becomes zero or converges to a very bad estimate in certain cases. The TS approach on the other hand is a pure Bayesian technique that keeps updating the belief and samples a new MDP from the updated samples. This, though converges, is only regret optima, so provides a very bad PAC-estimate. Thus, we showed that adding a better exploration bonus, can make the traditional TS sample efficient and converges to a {\it PAC-MDP}.

\section{CONCLUSION}	

In this work we propose TSEB - a Thompson sampling approach to model-based RL that uses an adaptive exploration bonus. This is the first TS variant that provides a PAC bound. We introduced a trade-off parameter that controls how much the exploration bonus influences the policy learnt on a sampled MDP. Tuning this parameter allows us to achieve better empirical performance with respect to the regret as well. While this work provides initial intuition into the PAC analysis of TS, more work needs to be done to establish a theory of useful exploration bonus and performance guarantees. Extending the model estimation to a non-parameterized setting, devoid of tight constraints over the parameter space, will also be an useful extension that will be applicable to a wide range of problems.

\newpage

%
%
\renewcommand\bibname{References}
\makeatletter
\makeatother
\bibliography{biblio}
\bibliographystyle{apalike}
\end{document}